\documentclass[11pt]{article}

\usepackage{amsthm,amsfonts,amsmath,amssymb,epsfig,color,float,graphicx,verbatim}

\makeatletter
\renewcommand{\section}{\@startsection{section}{1}{0pt}{-12pt}{5pt}{\large\bf}}
\renewcommand{\subsection}{\@startsection{subsection}{2}{0pt}{-12pt}{-5pt}{\normalsize\bf}}
\renewcommand{\subsubsection}{\@startsection{subsubsection}{3}{0pt}{-12pt}{-5pt}{\normalsize\bf}}
\makeatother

\usepackage{framed}

\def\nnewcolor{1}
\ifnum\nnewcolor=1

\fi
\ifnum\nnewcolor=0

\fi

\newtheorem{theorem}{Theorem}[section]
\newtheorem{proposition}{Proposition}[section]
\newtheorem{lemma}{Lemma}[section]
\newtheorem{corollary}{Corollary}[section]

\newtheorem{definition}{Definition}[section]
\newtheorem{observation}{Observation}[section]

\newcommand{\flattened}{{\mathrm{flat}}}

\newcommand{\R}{\mathbb{R}}

\newcommand{\E}{\mathbb{E}}

\newcommand{\dtv}{d_{\mathrm TV}}
\newcommand{\dk}{d_{\mathrm K}}
\newcommand{\wh}[1]{{\widehat{#1}}}

\newcommand{\rnote}[1]{\footnote{{\bf [[Rocco: {#1}\bf ]] }}}

\newcommand{\snote}[1]{\footnote{{\bf [[Siuon: {#1}\bf ]] }}}

\newcommand{\ignore}[1]{}

\newcommand{\eps}{\varepsilon}

\newcommand{\norm}[1]{\left\|#1\right\|}

\newcommand{\partit}{\mathcal}

\renewcommand{\eqref}[1]{Eq.~(\ref{#1})}

\newcommand{\eqdef}{\stackrel{{\mathrm {\footnotesize def}}}{=}}

\newcommand{\littlesum}{\mathop{\textstyle \sum}}
\newcommand{\littleprod}{\mathop{\textstyle \prod}}
\newtheorem{fact}[theorem]{Fact}

\newenvironment{algorithm}[1][\  ] %
{ \rm
\begin{tabbing}
....\=.....\=.....\=.....\=.....\=  \+ \kill
} %
{\end{tabbing} }

\floatstyle{ruled}
\newfloat{Algorithm}{H}{alg}
\newfloat{Subroutine}{H}{sub}

{
\begin{minipage}{1.0\linewidth} \begin{algorithm} %
} { \end{algorithm} \end{minipage} }

\title{Learning mixtures of structured distributions over discrete domains}

\author{Siu-On Chan\thanks{Supported by NSF award 
DMS-1106999, DOD ONR grant N000141110140 and NSF award CCF-1118083.}\\
UC Berkeley \\
{\tt siuon@cs.berkeley.edu}.\\
\and
Ilias Diakonikolas\thanks{This work was done while the author was at UC Berkeley supported by a Simons Postdoctoral Fellowship.}\\
University of Edinburgh\\
{\tt ilias.d@ed.ac.uk}.\\  
\and
Rocco A. Servedio\thanks{Supported by NSF grants CCF-0915929 and CCF-1115703.}\\
Columbia University\\
{\tt rocco@cs.columbia.edu}.\\
\and
Xiaorui Sun\thanks{Supported by NSF grant CCF-1149257.}\\
Columbia University \\
{\tt xiaoruisun@cs.columbia.edu}.
}

\begin{document}

\maketitle

\thispagestyle{empty}

\begin{abstract}

Let $\mathfrak{C}$ be a class of probability distributions over the discrete domain $[n] = \{1,\dots,n\}.$
We show that if $\mathfrak{C}$ satisfies a rather general condition -- essentially, that each distribution in
$\mathfrak{C}$ can be well-approximated by a variable-width
histogram with few bins -- then there is a highly efficient (both in terms of running time and sample complexity)
algorithm that can learn any mixture of $k$ unknown distributions from
$\mathfrak{C}.$

We analyze several natural types of distributions over $[n]$,
including log-concave, monotone hazard rate and unimodal distributions,
and show that they have the required structural property of being
well-approximated by a histogram with few bins.
Applying our general algorithm, we
obtain near-optimally efficient algorithms for all these mixture
learning problems as described below.  More precisely,

\begin{itemize}

\item {\bf Log-concave distributions:}  We learn any mixture of $k$
log-concave distributions over $[n]$ using $k \cdot
\tilde{O}(1/\eps^4)$ samples (independent of $n$) and running in time
$\tilde{O}(k \log(n) / \eps^4)$ bit-operations (note that reading a single
sample from $[n]$ takes $\Theta(\log n)$ bit operations).
For the special case $k=1$ we give an efficient
algorithm using $\tilde{O}(1/\eps^3)$
samples; this generalizes the main result of \cite{DDS12stoc} from the
class of Poisson Binomial distributions to the much broader class of all
log-concave distributions.  Our upper bounds are not far from
optimal since any algorithm for this learning problem requires
$\Omega(k/\eps^{5/2})$ samples.

\item {\bf Monotone hazard rate (MHR) distributions:}
We learn any mixture of $k$ MHR distributions over $[n]$ using
$O(k \log (n/\eps)/\eps^4)$ samples and running in time $\tilde{O}(k
\log^2(n) / \eps^4)$ bit-operations.  Any algorithm for this learning problem must use $\Omega(k \log(n)/\eps^3)$ samples.

\item {\bf Unimodal distributions:}
We give an algorithm that learns any mixture of $k$ unimodal distributions
over $[n]$ using $O(k \log (n)/\eps^{4})$ samples and running in time
$\tilde{O}(k \log^2(n) / \eps^{4})$ bit-operations.
Any algorithm for this problem must use $\Omega(k \log(n)/\eps^3)$ samples.

\end{itemize}

\ignore{

%%%%%%%%%%%%%%%%%%%%%%
START OF OLD ABSTRACT:
%%%%%%%%%%%%%%%%%%%%%%

We give time- and sample- efficient algorithms for learning mixtures of several
well-studied types of structured distributions over the discrete domain $[n]=\{1,\dots,n\}.$

We consider the following three
types of distributions:

\begin{itemize}

\item {\bf Log-concave distributions:}  We give an algorithm that learns any mixture of $k$ log-concave distributions over $[n]$ using $\tilde{O}(k/\eps^5)$ samples (independent of $n$) and running in time
    $BLAH$ bit-operations (note that reading a single sample from $[n]$ takes $\Theta(\log n)$ bit operations).
    This bound is not far from optimal since we also give an $\Omega(k/\eps^{5/2})$ sample lower bound for this
    learning problem.

\item {\bf Monotone hazard rate (MHR) distributions:}
We give an algorithm that learns any mixture of $k$ MHR distributions over $[n]$ using $\tilde{O}(k \log (n)/\eps^5)$ samples and running in time
    $BLAH$ bit-operations.
    We show that any algorithm for this learning problem must use $\Omega(k \log(n)/\eps^3)$ samples.

\item {\bf Unimodal distributions:}
We give an algorithm that learns any mixture of $k$ unimodal distributions over $[n]$ using $\tilde{O}(k \log (n)/\eps^5)$ samples and running in time
    $BLAH$ bit-operations.   We show that any algorithm for this learning problem must use $\Omega(k \log(n)/\eps^3)$
    samples.

\end{itemize}

All these results follow from a single unified general approach.

%%%%%%%%%%%%%%%%%%%%
END OF OLD ABSTRACT:
%%%%%%%%%%%%%%%%%%%%

}

\end{abstract}

\newpage

\section{Introduction}  \label{sec:intro}
%%%\vspace{-0.12cm}
\subsection{Background and motivation.} \label{ssec:background}
Learning an unknown probability distribution given access to
independent samples
is a classical topic with a long history in statistics and probability theory.
Theoretical computer science researchers have also been interested in
these problems at least since the 1990s \cite{KMR+:94, Dasgupta:99},
with an explicit focus on the \emph{computational efficiency}
of algorithms for learning distributions.
Many works in theoretical computer science
over the past decade have focused
on learning and testing various kinds of probability distributions
over high-dimensional spaces, see e.g. \cite{Dasgupta:99,
FreundMansour:99,DasguptaSchulman:00,AroraKannan:01,VempalaWang:02,FOS:05focs,
RubinfeldServedio:05,
BelkinSinha:10, KMV:10,MoitraValiant:10, ACS10}
and references therein.
There has also been significant recent
interest in learning and testing various types
of probability distributions over the discrete domain $[n]=\{1,\dots,n\}$,
see e.g.~\cite{BKR:04, ValiantValiant:11focs, ValiantValiant:11,DDS12soda,DDS12stoc}.

%%%\vspace{-0.12cm}
A natural type of distribution learning problem, which is the focus of
this work, is that of learning an
unknown \emph{mixture} of ``simple'' distributions.
Mixtures of distributions have received much attention
in statistics \cite{Lindsay:95,RednerWalker:84,TSM:85}
and in recent years have been intensively studied in computer
science as well (see many of the papers referenced above).
Given distributions $p_1,\dots,p_k$ and non-negative
values $\mu_1,\dots,\mu_k$ that sum to 1, we say that
$p = \sum_{i=1}^k \mu_i p_i$ is a \emph{$k$-mixture} of \emph{components}
$p_1,\dots,p_k$
with \emph{mixing weights} $\mu_1,\dots,\mu_k$. A draw from $p$ is obtained
by choosing $i \in [k]$ with probability $\mu_i$ and then making a draw
from $p_i$.

%%%\vspace{-0.12cm}
In this paper we work in essentially the classical
``density estimation'' framework \cite{Silverman:86,Scott:92,DL:01}
which is very similar to the model considered in \cite{KMR+:94} in
a theoretical computer science context.  In this framework the
learning algorithm is given access to independent samples drawn from an
unknown target distribution over $[n]$, and it must output a hypothesis
distribution $h$ over $[n]$ such that with high probability the
total variation distance $\dtv(p,h)$
between $p$ and $h$ is at most $\eps.$
Thus, for learning mixture distributions, our goal is
simply to construct a high-accuracy hypothesis distribution
which is very close to the mixture distribution that
generated the data.   In keeping with the spirit of \cite{KMR+:94}, we shall
be centrally concerned with the \emph{running time} as well as the
number of samples required by our algorithms that learn
mixtures of various types of discrete distributions over $[n].$

%%%\vspace{-0.12cm}
We focus on {\em density estimation} rather than, say,
clustering or parameter estimation, for several reasons.
First, clustering samples
according to which component in the mixture each sample came from is
often an impossible task unless
restrictive separation assumptions are made on the components; we prefer
not to make such assumptions.  Second, the classes
of distributions that we are chiefly interested in (such as log-concave, MHR
and unimodal distributions) are all non-parametric classes, so it is
unclear what ``parameter estimation'' would even mean for these
classes.  Finally, even in highly
restricted special cases, parameter estimation provably requires sample
complexity {\em exponential} in $k$, the number of components in the mixture.
Moitra and Valiant \cite{MoitraValiant:10} have
shown that parameter estimation for a mixture of $k$ Gaussians
inherently requires $\exp(\Omega(k))$ samples.  Their argument
can be translated to the discrete setting, with translated Binomial
distributions in place of Gaussians, to provide a similar lower bound
for parameter estimation of translated Binomial mixtures.  Thus, parameter
estimation even for a mixture of $k$ translated Binomial distributions over $[n]$
(a highly restricted special case of {\em all} the mixture classes we consider,
since translated Binomial distributions are log-concave, MHR and unimodal)
requires $\exp(\Omega(k))$ samples.
This rather discouraging lower bound motivates the study
of other variants of the problem of learning mixture
distributions.

\ignore{

OLD VERSION:

%Thus we are \emph{not}
%attempting to ``cluster'' samples according to which component
%each sample came from (often an impossible task unless restrictive
%assumptions are made on the components); instead, our goal is simply
%to construct a high-accuracy hypothesis distribution which is very close
%to the underlying mixture distribution that generated the data.\footnote{We
%note that in much recent work on learning mixtures of Gaussians
%the goal is to accurately estimate the \emph{parameters} of the
%component Gaussians (a more challenging task
%than density estimation).  In our work we only consider
%non-parametric classes of distributions (such as log-concave,
%MHR and unimodal distributions) so parameter estimation is not an option.}
%In keeping with the spirit of \cite{KMR+:94}, we shall
%be centrally concerned with the \emph{running time} as well as the
%number of samples required by our algorithms that learn
%mixtures of various types of discrete distributions over $[n].$

END OLD VERSION

}

Returning to our density estimation framework,
it is not hard to show that
from an information-theoretic perspective,
learning a mixture of distributions from a class $\mathfrak{C}$ of distributions
is never much harder than learning a single distribution
from $\mathfrak{C}$.
In Appendix~\ref{ap:generalmixtures} we
give a simple argument which establishes the following:
\begin{proposition} \label{obs:generalmixtures}{[Sample Complexity of Learning Mixtures]}
Let $\mathfrak{C}$ be a class of distributions over $[n]$.  Let $A$
be an algorithm which learns any unknown distribution $p$ in $\mathfrak{C}$
using $m(n,\eps) \cdot \log(1/\delta)$ samples, i.e.,
with probability $1 - \delta$ $A$
outputs a hypothesis distribution
$h$ such that $\dtv(p,h) \leq \eps$.
where $p \in \mathfrak{C}$ is the unknown target distribution.
Then there is an algorithm $A'$ which uses $\tilde{O}(k/\eps^3) \cdot m(n,\eps/20) \cdot \log^2(1/\delta)$ samples and
learns any unknown $k$-mixture of distributions in $\mathfrak{C}$ to variation distance $\eps$ with confidence probability $1 - \delta$.
\end{proposition}

%%%\vspace{-0.12cm}

While the generic algorithm $A'$ uses relatively few samples,
it is computationally
highly inefficient, with running time exponentially higher than the
runtime of algorithm $A$
(since $A'$ tries all possible partitions of its input sample
into $k$ separate subsamples).
Indeed, naive approaches to learning
mixture distributions run into a ``credit assignment'' problem
of determining which component distribution each sample point belongs to.
\ignore{
Indeed, from a computational perspective it seems quite challenging to
learn mixture distributions because of the ``credit assignment'' problem
of determining which component distribution each sample point belongs to.
}

%%%\vspace{-0.12cm}

As the main contributions of this paper,
we (i) give a general algorithm which \emph{efficiently}
learns mixture distributions over $[n]$ provided that the component
distributions satisfy a mild condition; and (ii)
show that this algorithm can be used to obtain highly efficient
algorithms for natural mixture learning problems.

%%%\vspace{-0.1cm}
\subsection{A general algorithm.}
The mild condition which we require of the component distributions in our
mixtures is essentially that each component distribution must be close to a
(variable-width) histogram with few bins.
More precisely, let us say that a distribution $q$ over $[n]$ is
\emph{$(\eps,t)$-flat} (see Section~\ref{sec:prelims})
if there is a partition of $[n]$ into $t$ disjoint
intervals $I_1,\dots,I_t$
such that $p$ is $\eps$-close (in total variation distance)
to the distribution obtained by ``flattening'' $p$ within each interval
$I_j$ (i.e., by replacing $p(k)$, for $k \in I_j$, with $\sum_{i \in I_j}p(i)/|I_j|$).
Our general result for learning mixture distributions is a highly efficient
algorithm that learns any $k$-mixture of $(\eps,t)$-flat distributions:
%%%\vspace{-0.1cm}
\begin{theorem}
[informal statement]
\label{thm:general}
There is an algorithm that learns any $k$-mixture
of $(\eps,t)$-flat distributions over $[n]$ to accuracy $O(\eps)$, using
$O(kt/\eps^3)$ samples and running in $\tilde{O}(k t \log(n)/\eps^3)$
bit-operations.
\end{theorem}
%%%\vspace{-0.1cm}
As we show in Section~\ref{sec:introapplic} below, Theorem~\ref{thm:general}
yields near-optimal sample complexity for a range of interesting
mixture learning problems, with a running time that is nearly linear in the sample size.
%bit-length of the input that it uses.
Another attractive feature of Theorem~\ref{thm:general} is that it always
outputs hypothesis distributions with a very simple structure (enabling
a succinct representation), namely histograms with at most $kt/\eps$ bins.
%%%\vspace{-0.1cm}
\subsection{Applications of the general approach.} \label{sec:introapplic}
We apply our general approach to obtain a wide range of learning results
for mixtures of various natural and well-studied types of discrete
distributions.  These include mixtures of \emph{log-concave} distributions,
mixtures of \emph{monotone hazard rate (MHR)} distributions, and mixtures of
\emph{unimodal} distributions.  To do this, in each case we need a structural result stating
that any distribution of the relevant type can be well-approximated by
a histogram with few bins.  In some cases (unimodal distributions) the necessary structural results were previously known, but in others
(log-concave and  MHR distributions) we establish novel structural results that, combined with our general approach,
yield nearly optimal algorithms.
%WAS: ourselves since we were not able to find such results in the literature.

\smallskip

\noindent {\bf Log-concave distributions.}  Discrete log-concave distributions
are essentially those distributions $p$ that satisfy
$p(k)^2 \geq p(k+1)p(k-1)$ (see Section~\ref{sec:logconcave} for
a precise definition).  They are closely analogous to log-concave distributions
over continuous domains, and encompass a range of interesting and well-studied
types of discrete distributions, including binomial, negative binomial,
geometric, hypergeometric, Poisson, Poisson Binomial, hyper-Poisson,
P\'{o}lya-Eggenberger, and Skellam distributions (see Section~1 of
\cite{BJR11}).  In the continuous setting, log-concave distributions 
include uniform, normal, exponential, logistic, extreme value,
Laplace, Weibull, Gamma, Chi and Chi-Squared and Beta distributions,
see \cite{BagnoliBergstrom05}.  Log-concave distributions over $[n]$ have
been studied in a range of different contexts including
economics, statistics and probability theory,
and algebra, combinatorics and geometry, see
\cite{An:95,BJR11,Stanley:89} and references therein.

Our main learning result for mixtures of
discrete log-concave distributions is:
%%\vspace{-0.1cm}
\begin{theorem} \label{thm:logconcave-informal}
There is an algorithm that learns any $k$-mixture of log-concave distributions
over $[n]$ to variation distance $\eps$ using $k\cdot \tilde{O}(1/\eps^{4})$
samples and running in $\tilde{O}(k \log(n)/\eps^{4})$ bit-operations.
\end{theorem}
%%\vspace{-0.1cm}
We stress that the sample complexity above is completely independent
of the domain size $n.$
In the special case of learning a single discrete log-concave distribution
we achieve an improved sample complexity of $\tilde{O}(1/\eps^3)$
samples, with running time $\tilde{O}(\log(n)/\eps^3)$.
This matches the sample complexity and running time of the
main result of \cite{DDS12stoc}, which was a specialized algorithm
for learning Poisson Binomial distributions over $[n]$.  Our new algorithm
is simpler, applies to the broader class of all log-concave
distributions, has a much simpler and more self-contained analysis,
and generalizes to mixtures of $k$ distributions (at the cost of
an additional $1/\eps$ factor in runtime and sample complexity).
%%%\vspace{-0.1cm}
We note that these algorithmic results are not far from the best possible for
mixtures of log-concave distributions.  We show in Section~\ref{sec:logconcave} that for $k
\leq n^{1 - \Omega(1)}$ and $\eps \geq 1/n^{\Omega(1)}$, any
algorithm for learning a mixture of $k$ log-concave distributions
to accuracy $\eps$ must use $\Omega(k/\eps^{2.5})$ samples.

\ignore{

%\begin{theorem} \label{thm:logconcave-lb-informal}
%Any algorithm that learns any $k$-mixture of log-concave distributions
%over $[n]$ to accuracy $\eps$ (for $k =BLAH(n)$)
%\rnote{Need to get a sharp statement here}
%must use $\Omega(k/\eps^{2.5})$ samples.
%\end{theorem}

}

\smallskip

\noindent {\bf Monotone Hazard Rate (MHR) distributions.}
A discrete distribution $p$ over $[n]$ is said to have a \emph{monotone (increasing)
hazard rate} if the \emph{hazard rate} $H(i) \eqdef {\frac {p(i)}{\littlesum_{j \geq i}
p(j)}}$ is a non-decreasing function of $i.$  It is well known
that every discrete log-concave distribution is MHR (see e.g.
part (ii) of Proposition~10 of \cite{An:95}), but
MHR is a more general condition than log-concavity (for example,
it is easy to check that every non-decreasing distribution over
$[n]$ is MHR, but such distributions need not be log-concave).  The MHR
property is a standard assumption in economics, in particular auction theory and mechanism design~\cite{Myerson:81, FT-book:91, mas1995microeconomic}. 
Such distributions also arise frequently in reliability theory;
\cite{BMP63} is a good reference for basic properties of these distributions.

Our main learning result for mixtures of MHR distributions is:
%%\vspace{-0.1cm}
\begin{theorem} \label{thm:mhr-informal}
There is an algorithm that learns any $k$-mixture of MHR distributions
over $[n]$ to variation distance $\eps$ using $O(k \log(n/\eps) /\eps^4)$ samples
and running in $\tilde{O}(k \log^2(n)/\eps^4)$ bit-operations.
\end{theorem}
%%\vspace{-0.1cm}
This theorem is also nearly optimal.  We show
that for $k \leq n^{1 - \Omega(1)}$ and $\eps \geq 1/n^{\Omega(1)}$, any
algorithm for learning a mixture of $k$ MHR distributions
to accuracy $\eps$ must use $\Omega(k\log(n)/\eps^{3})$ samples.

\ignore{

%\begin{theorem} \label{thm:mhr-lb-informal}
%Any algorithm that learns any $k$-mixture of MHR distributions
%over $[n]$ to accuracy $\eps$ (for \nnew{$k =o(n)$})
%must use $\Omega(k \log(n/k)/\eps^{3})$ samples.
%\rnote{\nnew{Need to check this carefully, see note in mhr section}}
%\end{theorem}

}

\smallskip

\noindent {\bf Unimodal distributions.}  A distribution over $[n]$
is said to be \emph{unimodal} if its probability mass function is monotone non-decreasing over
$[1,t]$ for some $t \leq n$ and then monotone non-increasing on
$[t,n]$.  Every log-concave distribution is unimodal, but the MHR
and unimodal conditions are easily seen to
be incomparable.  Many natural types of distributions are unimodal
and there has been extensive work on density estimation for unimodal distributions and related questions \cite{PrakasaRao:69,Wegman:70,BKR:04,Birge:97,
Fougeres:97}.

Our main learning result for mixtures of unimodal distributions is:
%%\vspace{-0.1cm}
\begin{theorem} \label{thm:unimodal-informal}
There is an algorithm that learns any $k$-mixture of unimodal distributions
over $[n]$ to variation distance $\eps$ using $O(k \log(n)/\eps^{4})$ samples
and running in $\tilde{O}(k \log^2(n)/\eps^{4})$ bit-operations.
\end{theorem}
%%\vspace{-0.1cm}
Our approach in fact extends to learning a $k$-mixture of $t$-modal
distributions (see Section~\ref{sec:unimodal}).
The same lower bound argument that we use for mixtures of MHR
distributions also gives us that for $k \leq n^{1 - \Omega(1)}$ and $\eps \geq 1/n^{\Omega(1)}$, any
algorithm for learning a mixture of $k$ unimodal distributions
to accuracy $\eps$ must use $\Omega(k\log(n)/\eps^{3})$ samples.

\ignore{

%This theorem is not far from the best possible result because of the
%following lower bound:

%\begin{theorem} \label{thm:unimodal-lb-informal}
%Any algorithm that learns any $k$-mixture of unimodal distributions
%over $[n]$ to accuracy $\eps$ (for \nnew{$k =o(n)$})
%must use $\Omega(k \log(n/k)/\eps^{3})$ samples.
%\end{theorem}

}

\subsection{Related work.}
\noindent {\bf Log-concave distributions:}
Maximum likelihood estimators for
both continuous~\cite{DumbgenRufibach:09,Walther09} and discrete~\cite{BJR11} log-concave distributions
have been recently studied by various authors.
For special cases of log-concave densities over $\R$ (that satisfy various restrictions on the shape of the pdf)
upper bounds on the minimax risk of estimators are known,
see e.g. Exercise 15.21 of \cite{DL:01}.
(We remark that these results do not imply the $k=1$ case of our log-concave mixture learning result.)
Perhaps the most relevant prior work
is the recent algorithm of \cite{DDS12stoc} which gives
a $\tilde{O}(1/\eps^3)$-sample, $\tilde{O}(\log(n)/\eps^3)$-time
algorithm for learning any Poisson Binomial Distribution over
$[n]$.  (As noted above, we match the performance of the \cite{DDS12stoc}
algorithm for the broader class of all log-concave distributions,
as the $k=1$ case of our log-concave mixture learning result.)

Achlioptas and McSherry \cite{AchlioptasMcSherry:05} and Kannan
et al. \cite{KSV08} gave algorithms for clustering points drawn
from a mixture of $k$ high-dimensional log-concave distributions,
under various separation assumptions on the distance between
the means of the components. %distributions in the mixture.
We are not aware of prior work on density estimation of mixtures of arbitrary log-concave distributions
in either the continuous or the discrete setting.
%(Note that even a mixture of two log-concave distributions
%may be far from being log-concave.)

\noindent
{\bf MHR distributions:}
As noted above, MHR distributions appear frequently and play an important role in reliability
theory and in economics (to the extent that the MHR condition is considered a standard assumption in these settings).
Surprisingly, the problem of learning an unknown MHR distribution or mixture of such distributions
has not been explicitly considered in the statistics literature. We note that several
authors have considered the problem of estimating the hazard rate of an MHR distribution in different contexts,
see e.g.~\cite{Wang-mhr, HuangWellner:93,GroeneboomJongbloed:11,Banerjee:08}.

\noindent
{\bf Unimodal distributions:}
The problem of learning a single unimodal distribution is well-understood:
Birg\'{e} \cite{Birge:97} gave an efficient algorithm for learning
continuous unimodal distributions (whose density is absolutely bounded); his algorithm, when
translated to the discrete domain $[n]$, requires $O(\log(n)/\eps^3)$ samples.
This sample size is also known to be optimal (up to constant factors)\cite{Birge:87}.
%since the result of Birg\'{e} \cite{Birge:87} implies that any algorithm for learning a monotone distribution
%over $[n]$ needs $\Omega(\log(n)/\eps^3)$ samples.
In recent work, Daskalakis et al. \cite{DDS12soda} gave an efficient
algorithm to learn $t$-modal distributions over $[n]$.
We remark that their result does not imply ours, as
%We note that a mixture of $t$ unimodal distributions need not be
%$t$-modal; in fact it is easy to see that
even a mixture of two unimodal distributions over $[n]$
may have $\Omega(n)$ modes. We are not aware of prior work on
efficiently learning mixtures of unimodal distributions.

\smallskip

\noindent {\bf Paper Structure.} Following some preliminaries in Section~\ref{sec:prelims}, Section~\ref{sec:main}
presents our general framework for learning mixtures. Sections~\ref{sec:logconcave},~\ref{sec:mhr} and~\ref{sec:unimodal}
analyze the cases of log-concave, MHR and unimodal mixtures respectively. 

\section{Preliminaries and notation}
\label{sec:prelims}
%%\vspace{-0.1cm}
We write $[n]$ to denote the discrete domain $\{1,\dots,n\}$ and $[i,j]$ to denote the set
$\{i,\dots,j\}$ for $i \leq j.$     For $v=(v(1),\dots,v(n)) \in \R^n$ we write
$\|v\|_1 = \sum_{i=1}^n |v(i)|$ to denote its $L_1$-norm.

%%\vspace{-0.1cm}

For $p$ a probability distribution over $[n]$
we write $p(i)$ to denote the probability of element $i \in [n]$ under $p$, so $p(i) \geq 0$ for all $i \in [n]$
and $\sum_{i=1}^n p(i)=1.$
For $S \subseteq [n]$ we write $p(S)$ to denote $\sum_{i \in S} p(i)$.
%For $S \subseteq [n]$, we write $p_S$ to denote the {\em conditional distribution}
%over $S$ that is induced by $p,$ i.e. $p_S(i) = {\frac {p(i)}{p(S)}}$ if $i \in S$ and $p_S(i)=0$ if $i \notin S.$
We write $p^S$ to denote the \emph{sub-distribution} over $S$
induced by $p$, i.e., $p^{S}(i) = p(i)$ if $i \in S$ and $p^S(i)=0$
otherwise.

%We use the  notation $P$ for the
%{\em cumulative distribution function (cdf)} corresponding to $p$, i.e.
%$P: [n] \to [0,1]$ is defined by $P(j) = \sum_{i=1}^j p(i)$.
%\rnote{Get rid of this sentence if notation not used later.}
%%\vspace{-0.1cm}
A distribution $p$ over $[n]$ is non-increasing (resp. non-decreasing)
if $p(i+1) \leq p(i)$ (resp. $p(i+1) \geq p(i)$), for all $i \in [n-1]$;
$p$ is \emph{monotone} if it is either non-increasing or non-decreasing.

%%\vspace{-0.1cm}
Let $p, q$ be distributions over $[n]$.
The {\em total variation distance} between $p$ and $q$ is
$\dtv (p,q) \eqdef \max_{S \subseteq [n]} \left| p(S) - q(S) \right|  = (1/2)\cdot \|p-q\|_1.$
The {\em Kolmogorov distance} between $p$ and $q$ is defined as
$\dk(p,q) \eqdef \max_{j \in [n]} \left| \littlesum_{i=1}^j p(i) - \littlesum_{i=1}^j q(i) \right|.$ Note that $\dk(p,q) \le \dtv (p,q).$

%%\vspace{-0.1cm}
Finally, the following notation and terminology will be useful:
given $m$ independent samples $s_1,\dots,s_m$, drawn from distribution
$p:[n] \to [0,1],$
the {\em empirical distribution} $\wh{p}_m : [n] \to [0,1]$ is defined as
follows: for all $i \in [n]$, $\wh{p}_m(i) = \left| \{j \in [m] \mid s_j=i\} \right| / m$.

\noindent {\bf Partitions, flat decompositions and refinements.}
\ignore{
We use calligraphic uppercase letters ${\cal P}, {\cal Q}, {\cal I}, {\cal J},$
etc. to denote partitions of the domain $[n]$ into disjoint intervals.
}
Given a partition ${\cal I} = \{I_1,\dots,I_t\}$ of $[n]$ into $t$ disjoint
intervals and a distribution $p$ over $[n]$,
we write $p^{\flattened(\cal I)}$ to denote
the {\em flattened distribution}.  This is
the distribution over $[n]$ defined as follows:  for $j \in [t]$
and $i \in I_j$, $p^{\flattened({\cal I})}(i) = p(I_j) / |I_j|$.
That is, $p^{\flattened({\cal I})}$ is obtained from $p$
by averaging the weight that $p$ assigns to each interval in ${\cal I}$ over the entire interval.
%The {\em reduced distribution} $p^{\reduced(\mathcal{I})}$ is the distribution over $[t]$
%that assigns the $i$th point the weight $p$ assigns to the
%interval $I_i$; i.e., for $i \in [t]$, we have $p^{\reduced(\mathcal{I})} (i)
%= p(I_i)$.\rnote{Get rid of reduced distribution from prelims if we
%don't use it.}
%%\vspace{-0.1cm}
\begin{definition}[Flat decomposition]
  Let $p$ be a distribution over $[n]$ and $\partit P$ be a partition of $[n]$
  into $t$ disjoint intervals.
  We say that $\partit P$ is a {\em $(p,\eps,t)$-flat decomposition of $[n]$} if
  $\dtv (p ,  p^{\flattened({\cal P})}) \leq \eps$. If there exists a $(p,\eps,t)$-flat decomposition of $[n]$ then we say
that $p$ is \emph{$(\eps,t)$-flat}.
\end{definition}

%We also define the flat subdistribution of $p$ over an interval $I$ as
%$\overline p^{\flattened(I)}(i) \triangleq p(I)/|I|$ for every $i \in I$.

%%\vspace{-0.1cm}
Let ${\cal I} = \{I_1,\dots,I_s\}$ be a partition of $[n]$ into $s$
disjoint intervals, and ${\cal J} = \{J_1,\dots,J_t\}$ be a partition
of $[n]$ into $t$ disjoint intervals.  We say that ${\cal J}$ is a
\emph{refinement} of ${\cal I}$ if each interval in ${\cal I}$ is a union
of intervals in ${\cal J}$, i.e., for every $a \in [s]$
there is a subset $S_a \subseteq [t]$ such that $I_a = \cup_{b \in S_a} J_b$.

%%\vspace{-0.1cm}

For ${\cal I}= \{I_i\}_{i=1}^r$ and ${\cal I}'=\{I'_i\}_{i=1}^s$ two
partitions of $[n]$ into $r$ and $s$ intervals respectively, we say that
the \emph{common refinement} of ${\cal I}$ and ${\cal I}'$ is the
partition ${\cal J}$ of $[n]$ into intervals obtained from ${\cal I}$
and ${\cal I}'$ in the obvious way, by taking all possible nonempty intervals
of the form $I_i \cap I'_j.$  It is clear that ${\cal J}$ is both a refinement
of ${\cal I}$ and of ${\cal I}'$ and that ${\cal J}$ contains at most
$r+s$ intervals.
%%\vspace{-0.1cm}

\subsection{Basic Tools.} \label{ssec:tools}
We recall some basic tools from probability.

\smallskip

\noindent {\bf The VC inequality.}
Given a family of subsets $\mathcal A$ over $[n]$, define $\norm p_{\mathcal A}
= \sup_{A\in \mathcal A} |p(A)|$.
The \emph{VC--dimension} of $\mathcal A$ is the maximum size of a subset
$X\subseteq [n]$ that is shattered by $\mathcal A$ (a set $X$ is shattered by
$\mathcal A$ if for every $Y \subseteq X$
some $A\in\mathcal A$ satisfies
$A\cap X = Y$).

\begin{theorem}[VC inequality, {\cite[p.31]{DL:01}}]
\label{thm:vc-inequality}
Let $\widehat{p}_m$ be an empirical distribution of $m$ samples from $p$.
Let $\mathcal A$ be a family of subsets of VC--dimension $d$.
Then
$$ \E \left [ \norm{p - \widehat{p}_m}_{\mathcal A} \right] \leq O(\sqrt{d/m}) .$$
\end{theorem}
\ignore{
%%\vspace{-0.1cm}
\noindent {\bf The DKW Inequality.}
The \emph{ Dvoretzky-Kiefer-Wolfowitz (DKW) inequality}
(\cite{DKW56}) can be obtained as a consequence of the
VC inequality taking ${\cal A}$ to be the class of all intervals
(though getting the sharp constant that we provide
in the version below takes additional work).
%It says that  $O(1/\eps^2)$ samples suffice to learn any distribution within error $\eps$ with respect to the \emph{Kolmogorov distance}.
The
DKW inequality states that for $m=\Omega((1/\eps^2)\cdot \ln(1/\delta))$,
with probability $1-\delta$
(over the draw of $m$ samples from $p$)
the {empirical distribution} $\wh{p}_m$ will be
$\eps$-close to $p$ in Kolmogorov distance.
This sample bound is asymptotically optimal and independent of the support size.
%%\vspace{-0.1cm}
\begin{theorem}[\cite{DKW56,Massart90}] \label{thm:DKW}
Let $\widehat{p}_m$ be an empirical distribution of $m$ samples from $p$.
Then for all $\eps>0$, we have
$$\Pr [ \dk(p, \wh{p}_m) > \eps ] \leq 2e^{-2m\eps^2}.$$
\end{theorem}
}
%%\vspace{-0.1cm}
\noindent {\bf Uniform convergence.}
We will also use the following uniform convergence bound:
%%\vspace{-0.1cm}
\begin{theorem}[{\cite[p17]{DL:01}}]
\label{thm:bdd-diff}
Let $\mathcal A$ be a family of subsets over $[n]$, and $\widehat{p}_m$ be an
empirical distribution of $m$ samples from $p$.
Let $X$ be the random variable $\norm{p - \hat p}_{\mathcal A}$.
Then we have
$$ \Pr\left[ X - \E[X] > \eta \right] \leq e^{-2m\eta^2}.$$
\end{theorem}

\section{Learning mixtures of $(\eps,t)$-flat distributions}\label{sec:main}

In this section we present and analyze
our general algorithm for learning mixtures of $(\eps,t)$-flat distributions.
We proceed in stages by considering three increasingly
demanding learning scenarios, each of which builds on the previous one.

\subsection{First scenario: known flat decomposition.}
% used to be:learning a flat distribution when the flat decomposition is known
\label{sec:firstalg}
We start with the simplest scenario, in which the learning algorithm
is given a partition ${\cal P}$ which is a $(p,\eps,t)$-flat
decomposition of $[n]$ for the target distribution $p$ being learned.

\begin{framed}
Algorithm \textsc{Learn-Known-Decomposition}$(p,{\cal P},\eps,\delta)$:

{\bf Input:}  sample access to unknown distribution $p$ over $[n]$;
$(p,\eps,t)$-flat
decomposition ${\cal P}$ of $[n]$; accuracy parameter $\eps$;
confidence parameter $\delta$

\begin{enumerate}
 \item Draw $m=O((t + \log 1/\delta)/\eps^2)$ samples to obtain an
    empirical distribution $\widehat{p}_m$.
%%\vspace{-0.1cm}
  \item Return $(\widehat{p}_m)^{\flattened(\mathcal{P})}$. %$\partit P$.
\end{enumerate}
%%\vspace{-0.2cm}
\end{framed}
%%\vspace{-0.1cm}
\begin{theorem}
\label{thm:learn-known-decomposition}
Let $p$ be any unknown target distribution over $[n]$ and ${\cal P}$
be any $(p,\eps,t)$-flat decomposition of $[n].$
Algorithm $\textsc{Learn-Known-Decomposition}(p,{\cal P},\eps,\delta)$
draws $O((t + \log(1/\delta))/\eps^2)$ samples from $p$
and with probability at least $1-\delta$,
outputs $(\widehat{p}_m)^{\flattened(\mathcal{P})}$
such that $\dtv((\widehat{p}_m)^{\flattened(\mathcal{P})},p)
\leq 2\eps$.
Its running time is
$\tilde{O}((t + \log(1/\delta)) \cdot \log(n) /\eps^2)$ bit operations.
\end{theorem}
%%\vspace{-0.1cm}
\begin{proof}
An application of the triangle inequality yields
$$\dtv \left( p, (\widehat{p}_m)^{\flattened(\mathcal{P})} \right) \leq \dtv \left( p, p^{\flattened(\mathcal{P})} \right) + \dtv \left( p^{\flattened(\mathcal{P})}, (\widehat{p}_m)^{\flattened(\mathcal{P})} \right).$$
The first term on the right-hand side is at most $\epsilon$ by the
definition of a $(p,\eps,t)$-flat decomposition. The second term is also at most $\epsilon$, as follows  by
Proposition \ref{prop:flattened-distance}, stated and proved below.
\end{proof}
%%\vspace{-0.1cm}
\begin{proposition}
\label{prop:flattened-distance}
Let $p$ be any distribution over $[n]$ and
let $\widehat{p}_m$ be an empirical distribution of
$m=\Theta((s+\log 1/\delta)/\eps^2)$ samples from $p$.
Let $\partit P$ be any partition of $[n]$ into at most $s$ intervals.
Then with probability at least $1-\delta$,
$\dtv (p^{\flattened({\partit P})},  (\widehat{p}_m)^{\flattened({\partit P})}) \leq \eps.$
\end{proposition}
%%\vspace{-0.1cm}
\begin{proof}
By definition we have
$$\dtv (p^{\flattened({\partit P})} , (\widehat{p}_m)^{\flattened({\partit P})}) = (1/2) \littlesum_{I \in {\cal P}} |p(I) - \widehat{p}_m(I)| = \|p - \widehat{p}_m \|_A,$$
where $A = \bigcup \{I\in \mathcal{P} \mid p(I) > \widehat{p}_m(I) \}$.
Since $\partit P$ contains at most $s$ intervals, $A$ is a union
of at most $s$ intervals.
Consequently the above right-hand side
is at most $\norm{p - \widehat{p}_m}_{\mathcal A_s}$, where $\mathcal A_s$
is the family of all unions of at most $s$ intervals over $[n]$.\footnote{Formally, define $\mathcal A_1 = \{ [a,b]
\mid 1 \leq a \leq b \leq n\} \cup \{\emptyset\}$ as the collection of
all intervals over $[n]$, including the empty interval.
Then $\mathcal A_s = \{I_1\cup \dots\cup I_s\mid I_1, \dots, I_s\in \mathcal A_1 \}$.}
Since the VC-dimension of $\mathcal A_s$ is $2s$,
Theorem \ref{thm:vc-inequality} implies that the considered quantity has expected value at most $\eps$.
The claimed result now follows by applying Theorem \ref{thm:bdd-diff} with
$\eta = \eps.$ 
\end{proof}
%%\vspace{-0.2cm}
\subsection{Second scenario:  unknown flat distribution.}
\label{sec:secondalg}
The second algorithm deals with the scenario in which the target
distribution $p$ is $(\eps/4,t)$-flat but no flat decomposition
is provided to the learner.  We show that in such a setting
we can construct a $(p,\eps,O(t/\eps))$-flat decomposition ${\cal P}$
of $[n]$, and then we can simply use this ${\cal P}$ to run
\textsc{Learn-Known-Decomposition}.

%%\vspace{-0.1cm}

The basic subroutine \textsc{Right-Interval} will be useful here (and later).
It takes as input an explicit description of a distribution $q$ over $[n]$,
an interval $J=[a,b] \subseteq [n]$, and a threshold $\tau>0.$
It returns the longest interval in $[a,b]$ that ends at $b$ and
has mass at most $\tau$ under $q$.
If no such interval exists then $q(b)$ must exceeds $\tau$, and the
subroutine simply returns the singleton interval $[b,b]$.
%%\vspace{-0.1cm}
\begin{framed}
%%\vspace{-0.1cm}
Subroutine \textsc{Right-Interval}$(q, J, \tau)$:
%%\vspace{-0.1cm}

{\bf Input:}  explicit description of distribution $q$; interval $J=[a,b]$;
threshold $\tau$

%%%\vspace{-0.1cm}
\begin{enumerate}
\item If $q(b) > \tau$ then set $i'=b$, otherwise set $i' = \min\{a \leq i\leq b\mid q([i, b]) \leq \tau \}$.
\item Return $[i',b]$.
\end{enumerate}
%%%\vspace{-0.1cm}
\end{framed}
%%\vspace{-0.1cm}
The algorithm to construct a decomposition is given below:
\begin{framed}
Algorithm \textsc{Construct-Decomposition}$(p,\tau,\eps,\delta)$:

{\bf Input:}  sample access to unknown distribution $p$ over
$[n]$; parameter $\tau$;\\
accuracy parameter $\eps$; confidence parameter $\delta$
\begin{enumerate}
%%\vspace{-0.1cm}
  \item Draw $m = O((1/\tau + \log 1/\delta)/\eps^2)$ samples to obtain an
    empirical distribution $\widehat{p}_m$.
%%\vspace{-0.1cm}
  \item Set $J = [n], \partit P = \emptyset$.
%%\vspace{-0.1cm}
  \item While $J\neq \emptyset$:
  %%\vspace{-0.1cm}
  \begin{enumerate}
%%\vspace{-0.1cm}
    \item Let $I$ be the interval returned by \textsc{Right-Interval}$(\widehat{p}_m, J, \tau)$.
%%\vspace{-0.1cm}
    \item Add $I$ to $\partit P$ and set $J = J\setminus I$.
  \end{enumerate}
%%\vspace{-0.1cm}
  \item Return ${\cal P}.$
  %%\vspace{-0.1cm}
\end{enumerate}
%%\vspace{-0.1cm}
\end{framed}
%%\vspace{-0.1cm}
\begin{theorem}\label{thm:construct-decomposition}
Let $\mathfrak{C}$ be a class of $(\eps/4,t)$-flat
distributions over $[n]$.
Then for any $p \in \mathfrak{C}$,
Algorithm~\textsc{Construct-Decomposition}$(p,\eps/(4t),\eps,\delta)$
draws $O(t/\epsilon^3 + \log(1/\delta)/\eps^2)$ samples from $p$, and
with probability at least $1 - \delta$ outputs a $(p,\eps,8t/\eps)$-flat
decomposition ${\cal P}$ of $[n]$.
Its running time is
$\tilde{O}((1/\tau + \log(1/\delta)) \cdot \log(n) /\eps^2)$ bit operations.
\end{theorem}
%%\vspace{-0.1cm}
To prove the above theorem we will need the following
elementary fact about refinements:

\begin{lemma} [{\cite[Lemma~4]{DDSVV13:soda}}]
\label{lem:DDSVV11}
Let $p$ be any distribution over $[n]$ and let
$\mathcal{I} = \{I_{i}\}_{i=1}^t$
be a $(p, \epsilon, t)$-flat decomposition of $[n]$. If $\mathcal{J} =
\{J_{i}\}_{i=1}^{t'}$ is a refinement of $\mathcal{I}$,
then $\mathcal{J}$ is a $(p, 2\epsilon, t')$-flat decomposition of $[n]$.
\end{lemma}
We will also use the following simple observation about the
\textsc{Right-Interval} subroutine:

%%\vspace{-0.1cm}

\begin{observation}
\label{obs:intervalmass}
Suppose \textsc{Right-Interval}$(q, J, \tau)$ returns an interval
$I\neq J$ and \textsc{Right-Interval}$(q, J\setminus I, \tau)$
returns $I'$.
Then $q(I) + q(I') > \tau$.
\end{observation}

%%%\vspace{-0.2cm}

\begin{proof}[Proof of Theorem \ref{thm:construct-decomposition}]
Let $\tau \eqdef \eps/(4t).$
By Observation \ref{obs:intervalmass}, the partition ${\cal P}$ that
the algorithm constructs must contain at most $2/\tau$ intervals.
Let $\partit Q$ be the common refinement of $\partit P$ and a
$(p, \epsilon/4, t)$-flat decomposition of $[n]$ (the existence of
such a decomposition is guaranteed because every distribution in
$\mathfrak{C}$ is $(\eps/4,t)$-flat).
Now note that $$\dtv (p, p^{\flattened(\partit P)}) \leq \dtv(p, p^{\flattened({\cal Q})}) + \dtv(p^{\flattened({\cal Q})}, p^{\flattened({\cal P})}).$$
Since $\partit Q$ is a refinement of the $(p, \epsilon/4, t)$-flat
decomposition of $[n]$, Lemma~\ref{lem:DDSVV11} implies that the first term
on the RHS is at most $\eps/2$.
It remains to bound
$\Delta = \dtv(p^{\flattened({\cal Q})}, p^{\flattened({\cal P})}).$
Fix any interval $I \in {\cal P}$ and let us
consider the contribution
$$(1/2) \littlesum_{j \in I} \big| p^{\flattened({\cal Q})}(j) - p^{\flattened({\cal P})}(j) \big|$$
of $I$ to $\Delta$.
If $I\in \partit P\cap \partit Q$ then the contribution to $\Delta$ is zero;
on the other hand, if $I \in {\cal P} \setminus {\cal Q}$ then the
contribution to $\Delta$ is at most $p(I)/2$. Thus the total contribution summed across all $I \in {\cal P}$ is
at most $(1/2) \littlesum_{I\in \partit P\setminus \partit Q} p(I)$.
Now we observe that with probability at least $1-\delta$ we have
%%\vspace{-0.1cm}
\begin{equation}
\label{eqn:p-minus-q}
(1/2) \littlesum_{I\in \partit P\setminus \partit Q} p(I) \leq \eps/4 + (1/2) \littlesum_{I\in \partit
  P\setminus \partit Q} \widehat{p}_{m}(I),
\end{equation}
where the inequality follows from the fact that
$\dtv \left( p^{\flattened({\cal P})} , (\widehat{p}_m)^{\flattened({\cal P})} \right) \leq \eps/4$
by Proposition~\ref{prop:flattened-distance}.
If $I\in \partit P\setminus \partit Q$ then $I$ cannot be a singleton, and
hence $\widehat{p}_m(I) \leq \tau$ by definition of \textsc{Right-Interval}.
Finally, it is easy to see that at most $t$
intervals $I$ in $\partit P$ do not belong to $\partit Q$
(because ${\cal Q}$ is the common refinement of ${\cal P}$
and a partition of $[n]$ into at most $t$ intervals).
Thus the second term on RHS of \eqref{eqn:p-minus-q} is at most $t\tau =
\eps/4$.
Hence $\Delta \leq \eps/2$ and the theorem is proved.
\end{proof}
Our algorithm to learn an unknown $(\eps/4,t)$-flat distribution is now very simple:
\begin{framed}
Algorithm \textsc{Learn-Unknown-Decomposition}$(p,t,\eps,\delta)$:

{\bf Input:}  sample access to unknown distribution $p$ over $[n]$; parameter $t$; \\
accuracy parameter $\eps$; confidence parameter $\delta$
%%\vspace{-0.1cm}
\begin{enumerate}
%%\vspace{-0.1cm}
  \item Run \textsc{Construct-Decomposition}$(p,\eps/(4t),\eps,\delta/2)$ to obtain a $(p,\eps,8t/\eps)$-flat decomposition ${\cal P}$ of $[n]$.
%%\vspace{-0.1cm}
  \item Run \textsc{Learn-Known-Decomposition}$(p,{\cal P},\eps,\delta/2)$ and return the hypothesis $h$ that it outputs.
%%\vspace{-0.1cm}
\end{enumerate}
%%\vspace{-0.2cm}
\end{framed}
%%\vspace{-0.1cm}
The following is now immediate:
\begin{theorem}\label{thm:learn-unknown-decomposition}
Let $\mathfrak{C}$ be a class of $(\eps/4,t)$-flat
distributions over $[n]$.
Then for any $p \in \mathfrak{C}$,
Algorithm~\textsc{Learn-Unknown-Decomposition}$(p,t,\eps,\delta)$
draws $O(t/\epsilon^3 + \log(1/\delta)/\eps^2)$ samples from $p$, and
with probability at least $1 - \delta$ outputs a hypothesis distribution $h$
satisfying $\dtv(p,h) \leq \eps.$
Its running time is
$\tilde{O}(\log(n) \cdot (t/\eps^3 + \log(1/\delta)/\eps^2))$ bit operations.
\end{theorem}
%%\vspace{-0.1cm}
\subsection{Main result (third scenario): learning a mixture of flat distributions.}
\label{sec:mainalg}
%%\vspace{-0.1cm}
We have arrived at the scenario of real interest to us, namely learning
an unknown mixture of $k$ distributions each of which has
an (unknown) flat decomposition.  The key to learning such distributions
is the following structural result, which says that any such mixture
must itself have a flat decomposition:

\begin{lemma}\label{lem:mixture}
Let $\mathfrak{C}$ be a class of $(\eps,t)$-flat distributions over
$[n]$, and let $p$ be any $k$-mixture of distributions in
$\mathfrak{C}.$
Then $p$ is a $(2 \eps, kt)$-flat distribution.
\end{lemma}
\begin{proof}
Let $p=\sum_{j=1}^k \mu_j p_j$ be a $k$-mixture of components
$p_1,\dots,p_k \in \mathfrak{C}.$  Let ${\cal P}_j$ denote
the $(p_j, \epsilon, t)$-flat decomposition of $[n]$ corresponding
to $p_j$,
and let $\mathcal{P}$ be the common
refinement of $\mathcal{P}_1, \mathcal{P}_2, \dots, \mathcal{P}_k$.
It is clear that $\mathcal{P}$ contains at most $kt$ intervals.
By Lemma \ref{lem:DDSVV11}, $\mathcal{P}$ is a $(p_j, 2\epsilon, kt)$-flat
decomposition for every $p_j$. Hence we have
\begin{eqnarray}
\dtv \left( p , p^{\flattened({\partit P})} \right) &=& \dtv \left( \littlesum_{j=1}^k \mu_j p_j, \littlesum_{j=1}^k \mu_j (p_j)^{\flattened(\partit P)} \right) \nonumber \\ 
&\leq& \littlesum_{j=1}^k \mu_j \dtv \left( p_j, (p_j)^{\flattened({\cal P})} \right) \label{eqn:triangle}\\ 
&\leq& 2\eps \label{eqn:cc}
\end{eqnarray}
where (\ref{eqn:triangle}) is the triangle inequality and (\ref{eqn:cc}) follows from the fact that the expression in (\ref{eqn:triangle}) is 
a nonnegative convex combination of terms bounded from above by $2\eps$.
\end{proof}

Given Lemma \ref{lem:mixture}, the desired mixture learning algorithm
follows immediately from the results of the previous subsection:

\begin{corollary}[see Theorem~\ref{thm:general}]
\label{cor:main}
Let $\mathfrak{C}$ be a class of $(\eps/8,t)$-flat distributions over $[n]$,
and let $p$ be any $k$-mixture of distributions in $\mathfrak{C}$.
Then Algorithm~\textsc{Learn-Unknown-Decomposition}$(p,kt,\eps,\delta)$
draws $O(kt/\eps^3 + \log(1/\delta)/\eps^2)$ samples from $p$,
and with probability at least $1 - \delta$ outputs a hypothesis
distribution $h$ satisfying $\dtv(p,h) \leq \eps.$
Its running time is
$\tilde{O}(\log(n) \cdot (kt/\eps^3 + \log(1/\delta)/\eps^2))$
bit operations.
\end{corollary}

\section{Learning mixtures of log-concave distributions}
\label{sec:logconcave}

%%\vspace{-0.1cm}
In this section we apply our general method from Section~\ref{sec:main}
to learn \emph{log-concave} distributions  over $[n]$ and mixtures of
such distributions. We start with a formal definition:

%%\vspace{-0.1cm}

\begin{definition} \label{def:logconcave}
A probability distribution $p$ over $[n]$ is said to be \emph{log-concave}
if it satisfies the following conditions:
(i) if $1 \leq i < j < k \leq n$ are such that $p(i)p(k)>0$ then
$p(j) > 0$; and (ii) $p(k)^2 \geq p(k-1)p(k+1)$ for all $k \in [n].$
\end{definition}

We note that while some of the literature on discrete log-concave distributions
states that the definition consists solely of item (ii) above, item (i)
is in fact necessary as well since without it log-concave distributions
need not even be unimodal (see the discussion following
Definition 2.3 of \cite{BJR11}).

In Section~\ref{sec:logconcave-decomposition} we give an efficient algorithm
which constructs an $(\eps,O(\log(1/\eps)/\eps))$-flat decomposition of
any target log-concave distribution.  Combining this with Algorithm
\textsc{Learn-Known-Decomposition} we obtain an $\tilde{O}(1/\eps^3)$-sample
algorithm for learning a single discrete log-concave distribution, and
combining it with Corollary~\ref{cor:main} we obtain a $k \cdot
\tilde{O}(1/\eps^4)$-sample
algorithm for learning a $k$-mixture of log-concave distributions.
%We give lower bounds for these learning problems at the end of Section~\ref{sec:logconcave}.
%%\vspace{-0.15cm}
\subsection{Constructing a flat decomposition given samples from a
log-concave distribution.}
\label{sec:logconcave-decomposition}
We recall the well-known fact that
log-concavity %of a distribution $p$ over $[n]$
implies unimodality %of $p$
(see e.g.~\cite{KeilsonGerber:71}). Thus, it is useful to analyze
log-concave distributions which additionally are monotone
(since a general log-concave distribution can be viewed as
consisting of two such pieces).
With this motivation we give the following lemma:
\begin{lemma}
\label{lemma:lcc-uniform}
Let $p$ be a distribution over $[n]$ that is non-decreasing and log-concave on
$[1,b] \subseteq [n]$.
Let $I = [a,b]$ be an interval of mass $p(I) = \tau$, and suppose that
the interval $J = [1,a-1]$ has mass $p(J) = \sigma > 0$.
Then
$$p(b)/p(a) \leq 1 + \tau/\sigma.$$
\end{lemma}
\begin{proof}
Let $s \eqdef |I| = b-a+1$ be the length of $I$.
We decompose $J$ into intervals $J_1, \dots, J_t$ of length $s$, starting from
the right.
More precisely, $$J_j \eqdef I - js = [a-js,b-js]$$ for $1 \leq j \leq t \eqdef \lceil(a-1)/s\rceil.$
The leftmost interval $J_t$ may contain non-positive integers;
for this reason define $p(i) \eqdef 0$ for non-positive $i$
(note that the new distribution is still log-concave).
Also define $J_0 \eqdef I = [a,b]$.
%%%\vspace{-0.1cm}
Let $\lambda \eqdef p(b)/p(a)$.
We claim that
%%\vspace{-0.1cm}
\begin{equation}
  \label{eqn:exppointmass}
  p(i-s) \leq (1/\lambda) \cdot p(i)
\end{equation}
%%%\vspace{-0.1cm}
for $1 \leq i\leq b$.
\eqref{eqn:exppointmass} holds for $i = b$, since $p(b-s) \leq p(a)$ by the
non-decreasing property.
The general case $i\leq b$ follows by induction and using the
fact that the ratio $p(i-1)/p(i)$ is non-decreasing in $i$ for any
log-concave distribution (an immediate consequence of the definition of
log-concavity).

It is easy to see that \eqref{eqn:exppointmass} implies
%%\vspace{-0.1cm}
$$p(J_{j+1}) \leq (1/\lambda) \cdot p(J_j)$$%%\vspace{-0.1cm}
for $0 \leq j \leq t$.
Since the intervals have geometrically decreasing mass, this implies that
$$\sigma = \littlesum_{1\leq j\leq t} p(J_j) \leq p(I) \littlesum_{j\geq 1} \lambda^{-j} =
  \frac \tau{\lambda-1} .$$
Rearranging yields the desired inequality.
\end{proof}

We will also use the following elementary fact:
\begin{fact}
\label{obs:multi-uniform}
Let $p$ be a distribution over $[n]$ and $I \subseteq [n]$ be an interval such that
$ \max_{i, j \in I} p(i)/p(j) \leq 1+\eta$ 
(i.e., $p$ is $\eta$-multiplicatively close to uniform over the interval $I$).
Then the flattened sub-distribution $p^{\flattened(I)}(i) \eqdef p(I)/|I|$ satisfies
$\dtv (p^I , p^{\flattened(I)}) \leq \eta \cdot p(I) .$
\end{fact}

We are now ready to present and analyze our algorithm
\textsc{Decompose-Log-Concave} that draws samples from an
unknown log-concave distribution and outputs a flat decomposition.
The algorithm simply runs the general algorithm
\textsc{Construct-Decomposition} with an appropriate choice of
parameters.  However the analysis will not go via the ``generic''
Theorem~\ref{thm:construct-decomposition} (which would yield a weaker
bound) but instead uses Lemma~\ref{lemma:lcc-uniform}, which is specific
to log-concave distributions.
\begin{framed}
%%\vspace{-0.1cm}
Algorithm \textsc{Decompose-Log-Concave}$(p,\eps,\delta)$:

%%\vspace{-0.1cm}

{\bf Input:}  sample access to unknown log-concave distribution $p$
over $[n]$; \\
accuracy parameter $\eps$; confidence parameter $\delta$

\begin{enumerate}
\item Set $\tau = \Theta(\eps/\log(1/\eps))$. 
\item Run \textsc{Construct-Decomposition}$(p,\tau,\eps,\delta)$ and
return the decomposition ${\cal P}$ that it yields.
\end{enumerate}
%%\vspace{-0.1cm}
\end{framed}
Our main theorem in this section is the following:
\begin{theorem}
\label{thm:lcc-dec}
For any log-concave distribution $p$ over $[n]$,
Algorithm~\textsc{Decompose-LogConcave}$(p,\eps,\delta)$
draws $O(\log(1/\eps)/\eps^3 + \log(1/\delta)(\log(1/\eps))^2/
\eps^2)$ samples
from $p$ and with probability at least $1-\delta$ constructs a decomposition
$\partit P$ that is $(p, \eps, O(\log(1/\eps)/\eps))$-flat.
\end{theorem}

\begin{proof}
We first note that the number of intervals in $\partit P$ is at most
$2/\tau$ by Observation \ref{obs:intervalmass}; this will be useful
below.\ignore{
Indeed, if \textsc{Right-Interval} is called more than $2/\tau$ times, then the
first $2/\tau$ intervals returned from the subroutine will have total mass
exceeding $1$ (under $\widehat{p}_m$), a contradiction.}
We may also assume that $\dk(p,\widehat{p}_m) \leq \tau$,
where $\widehat{p}_m$ is the empirical distribution obtained in Step~1
of {\textsc{Construct-Decomposition}};
this inequality holds with probability at least $1-\delta$,
as follows by a combined application of Theorems~\ref{thm:vc-inequality} and~\ref{thm:bdd-diff}.
%DKW inequality.
Since $p$ is log-concave, it is unimodal.
Let $i_0$ be a mode of $p$.

Let $\partit P_L = \{ I\in \partit P\mid I\subset [1,i_0-1] \}$ be the
collection of intervals to the left of $i_0$.
We now bound the contribution of intervals in $\partit P_L$ to
$\Delta \eqdef \dtv \left( p , p^{\flattened({\partit P})} \right)$.
Let $I_1, \dots, I_{t_L}$ be the intervals in $\partit P_L$ listed from left to
right.
Let $J_j = \cup_{j' < j} I_{j'}$ be the union of intervals to the left of
$I_j$.
If $I_j$ is a singleton, its contribution to $\Delta$ is zero.
Otherwise,
%%\vspace{-0.1cm}
$$p(I_j) \leq \widehat{p}_{m}(I_j) + \tau \leq 2\tau$$
by the $\tau$-closeness of $p$ and $\widehat{p}_m$
in Kolmogorov distance and the definition of
\textsc{Right-Interval}.
Also, by Observation \ref{obs:intervalmass}, $\widehat{p}_m(J_{j}) \geq \lfloor
(j-1)/2\rfloor \tau \geq ((j-1)/2 - 1)\tau$, and hence
%%\vspace{-0.1cm}
$$p(J_j) \geq \widehat{p}_{m}(J_j) - \tau \geq \tau(j-5)/2,$$
again by closeness in Kolmogorov distance.

Since $p$ is non-decreasing on $[1,i_0-1]$, we have
%%\vspace{-0.1cm}
$$\left\| p^{I_j} - p^{\flattened(I_j)} \right\|_1 \leq \frac 8 {j-5} \tau $$
for $j > 5$, by Lemma \ref{lemma:lcc-uniform} and Fact
\ref{obs:multi-uniform}, using the upper and lower bounds
on $p(I_j)$ and $p(J_j)$ respectively.
Consequently, $\| p^{I_j} -p^{\flattened(I_j)}\|_1 \leq O(\tau/j)$ for all $j\in [t_L]$.
Summing this inequality, we get
%%%\vspace{-0.1cm}
$$\littlesum_{j\leq t_L} \| p^I - p^{\flattened(I_j)}\|_1 \leq \littlesum_{j\leq t_L}
  O(\tau/j) = O(\tau\log(1/\tau)) .$$
The right-hand side above is at most $\eps/2$ by our choice of $\tau$ (with an
appropriate constant in the big-oh).

Similarly, let $\partit P_R = \{ I\in \partit P\mid I\subset [i_0+1, n] \}$ be
the collection of intervals to the right of $i_0$.
An identical analysis
(using the obvious analogue of Lemma~\ref{lemma:lcc-uniform} for
non-increasing log-concave distributions on $[i_0+1,n]$)
shows that the contribution of intervals in $\partit P_R$
to $\Delta$ is at most $\eps/2$.

Finally, let $I_0\in \partit P$ be the interval containing $i_0$.
If $I_0$ is a singleton, it does not contribute to $\Delta$.
Otherwise, $\widehat{p}_{m}(I_0) \leq \tau$ and $p(I_0) \leq 2\tau$,
hence the contribution of $I_0$ to $\Delta$ is at most $2\tau$.

Combining all three cases, $$\left\| p - p^{\flattened({\partit P})}\right\|_1 \leq \eps/2 + \eps/2
+ 2\tau \leq 2\eps.$$
Hence $d_\text{TV} \left( p, p^{\flattened({\partit P})}  \right) \leq \eps$
as was to be shown.
\end{proof}

%%\vspace{-0.1cm}
Our claimed upper bounds follow from the above theorem by using our framework of Section~\ref{sec:main}.
Indeed, it is clear that we can learn any unknown log-concave distribution
by running Algorithm~\textsc{Decompose-Log-Concave}$(p,\eps,\delta/2)$
to obtain a decomposition ${\cal P}$ and then Algorithm~\textsc{Learn-Known-Decomposition}$(p,P,\eps,\delta/2)$
to obtain a hypothesis distribution $h$:
%%\vspace{-0.1cm}
\begin{corollary}
\label{cor:learn-logconcave}
Given sample access to a log-concave distribution $p$ over $[n]$,
there is an algorithm \textsc{Learn-Log-Concave}$(p,\eps,\delta)$ that uses
$O(\log(1/\delta)\log(1/\eps)/\eps^3)$ samples from $p$ and
with probability at least $1 - \delta$ outputs a distribution $h$
such that $\dtv(p,h) \leq \eps.$
Its running time is
$\tilde{O}(\log(n) \cdot (1/\eps^3 + \log(1/\delta)/\eps^2))$
bit operations.
\end{corollary}
Theorem~\ref{thm:lcc-dec} of course implies that
every log-concave distribution $p$ is $(\eps,O(\log(1/\eps)/\eps))$-flat.
We may thus apply Corollary~\ref{cor:main} and obtain our main
learning result for $k$-mixtures of log-concave distributions:
\begin{corollary} [see Theorem~\ref{thm:logconcave-informal}]
\label{cor:learn-logconcave-mixture}
Let $p$ be any $k$-mixture of log-concave distributions over $[n]$.
There is an algorithm \textsc{Learn-Log-Concave-Mixture}$(p,k,\eps,\delta)$
that draws $O(k\log(1/\eps)/\eps^4 + \log(1/\delta)/\eps^2)$ samples from
$p$ and with probability at least $1-\delta$ outputs a
distribution $h$ such that $\dtv(p,h) \leq \eps.$
Its running time is
$\tilde{O}(\log(n) \cdot (k\log(1/\eps)/\eps^4 + \log(1/\delta)/\eps^2))$
bit operations.
\end{corollary}

\smallskip

\noindent
{\bf Lower bounds.}
It is shown in~\cite[Lemma 15.1]{DL:01} that learning
a continuous distribution whose density is bounded and convex over $[0,1]$ to accuracy $\eps$
requires $\Omega((1/\eps)^{5/2})$ samples. An easy adaptation of this argument
implies the same result for a bounded concave density over $[0,1]$. By an appropriate discretization
procedure, one can show that learning a discrete concave density over $[n]$ requires $\Omega((1/\eps)^{5/2})$ samples
for all $\eps \geq 1/n^{\Omega(1)}.$ Since a discrete concave distribution is also log-concave, the same lower bound
holds for this case too. For the case of $k$-mixtures,
we may consider a uniform mixture of $k$
component distributions where the $i$-th distribution in the mixture is supported on $[1+(i-1)n/k,in/k]$ and is log-concave on its support.
It is clear that each component distribution is log-concave over $[n]$, and it is not difficult to see that in
order to learn such a mixture to accuracy $\eps$, at least
$9/10$ of the component distributions must be learned to total
variation distance at most $10 \eps$.  We thus get that
for
$k \leq n^{1-\Omega(1)}$
and
$\eps \geq 1/n^{\Omega(1)}$,
any algorithm for learning a mixture of $k$ log-concave distributions to
accuracy $\eps$
must use $\Omega(k \eps^{-5/2})$ samples.

\section{Learning mixtures of MHR distributions} \label{sec:mhr}

%%\vspace{-0.1cm}
In this section we apply our general method from Section~\ref{sec:main}
to learn \emph{monotone hazard rate (MHR)} distributions  over $[n]$ and mixtures of
such distributions.
%%\vspace{-0.1cm}
\begin{definition}
Let $p$ be a distribution supported in $[n]$. The \emph{hazard rate}
of $p$ is the function
$H(i) \eqdef {\frac {p(i)}{\littlesum_{j \geq i}
p(j)}}$; if $\littlesum_{j \geq i} p(j) = 0$ then we say $H(i) = +\infty.$
We say that $p$ has \emph{monotone hazard rate} (MHR) if
$H(i)$ is a non-decreasing function over $[n].$
\end{definition}
%%\vspace{-0.1cm}
It is known that every log-concave distribution over $[n]$ is MHR but
the converse is not true, as can easily be seen from the fact that every
monotone non-decreasing distribution over $[n]$ is MHR.

In Section~\ref{sec:learnsinglemhr} we prove that every MHR distribution
over $[n]$ has an $(\eps,O(\log(n/\eps)/\eps))$-flat decomposition.
We combine this with our general results from
Section~\ref{sec:main} to get learning results for mixtures of MHR
distributions.
%%\vspace{-0.1cm}
\subsection{Learning a single MHR distribution.}
\label{sec:learnsinglemhr}

\ignore{

%Mention that for an MHR monotone non-increasing distribution we can learn
%in $\tilde{O}(1/\eps^3)$ samples
%using the same algorithm as log-concave.  But in general we can't do this well;
%observe that any monotone non-decreasing distribution is MHR, and mention
%the resulting lower bound.  Say
%that we do almost this well.

}

%\begin{proposition}[Birg\'e decomposition, {\cite[Theorem 5]{DDSVV12}}]
%\label{prop:birge}
%Some decomposition $\partit P$ of $[n]$ is $(p, \eps, O(\eps^{-1}\log
%n))$-flat\footnote{[[Siuon: Looking at the proof in Appendix E of the Testing
%$k$-modal paper, I think the number of intervals is $\eps^{-1}\log(n+1)$ rather
%than $\eps^{-1}\log(\eps n+1)$.]]} for any non-decreasing distribution $p$.
%Further, $\partit P$ can be constructed in time $\poly(\log n, 1/\eps)$.
%\end{proposition}

%\begin{proposition}[Robustness, {\cite[Claim 10]{DDS12soda}}]
%\label{prop:robust}
%Let $p$ and $q$ be distributions such that $d_\text{TV}(p, q) \leq \tau$.
%Then any $(p, \eps, t)$-flat decomposition $\partit P$ is also $(q, 2\tau +
%\eps, t)$-flat.
%\end{proposition}

Our algorithm to construct a flat decomposition of an MHR distribution $p$ is
\textsc{Decompose-MHR}, given below.  Note that this algorithm
takes an explicit description of $p$ as input and does not draw any
samples from $p$.
Roughly speaking, the algorithm works by partitioning
$[n]$ into intervals such that within each interval
the value of $p$ never deviates from its value at the leftmost point
of the interval by a multiplicative factor of more than
$(1 + \eps/8).$
\begin{framed}
%%\vspace{-0.1cm}
Algorithm \textsc{Decompose-MHR}$(p,\eps)$:

{\bf Input:}  explicit description of MHR distribution $p$ over $[n]$;
accuracy parameter $\eps > 0$
\begin{enumerate}
%\item Draw $O(\log(1/\delta)/\eps^4)$ samples to obtain an empirical distribution $\hat p$.
\item Set $J = [n]$ and initialize $\partit Q$ to be the empty set.
\item   Let $I$ be the interval returned by
\textsc{Right-Interval}$({p}, J, \eps/8)$, and
$I'$ be the interval returned by
\textsc{Right-Interval}$({p}, J\setminus I,  \eps/8)$.
Set $J = J \setminus (I \cup I')$.
\item Set $i \in J$ to be the smallest integer such that $p(i) \geq  \eps/(4n)$.  If no such $i$ exists, let $I'' = J$ and
go to Step 5. Otherwise, let $I'' = [1, i-1]$ and $J = J \setminus I''$.
\item While $J \neq \emptyset$:
\begin{enumerate}
\item Let $j\in J$ be the smallest integer such that either $p(j) > (1+ \eps/8)p(i)$ or $p(j) < \frac{1}{1+ \eps/8}p(i)$ holds.
If no such $j$ exists let $I''' = J$, otherwise, let $I''' = [i, j-1]$.
\item Add $I'''$ to $\partit Q$, and set $J = J\setminus I'''$.
\item Let $i = j$.
\end{enumerate}
\item Return $\partit P = \partit Q \cup \{I, I', I''\}$.
\end{enumerate}
\end{framed}
%\rnote{In Step~5 of the algorithm it had been ``$\partit P = \partit Q \cup \{I, I', I''\}$'' but I think $I''$ should not be there.}
Our first lemma for the analysis of \textsc{Decompose-MHR} states that
MHR distributions satisfy a condition that is similar to
being monotone non-decreasing:
\begin{lemma}
\label{lemma:mhr-monotone}
Let $p$ be an MHR distribution over $[n]$.
Let $I = [a,b] \subset [n]$ be an interval, and $R = [b+1, n]$ be the
elements to the right of $I$.
Let $\eta \eqdef p(I)/p(R)$.
Then $p(b+1) \geq \frac{1}{1+\eta}p(a) $.
\end{lemma}

\begin{proof}
If $p(b+1) > p(a)$ then $p(b+1) \geq {\frac 1 {1+\eta}}p(a)$
holds directly, so for the rest of the proof
we may assume that $p(b+1) \leq p(a)$.

By the definition of the MHR condition we have $\frac{p(a)}{p([a+1,n])} \leq \frac{p(b+1)}{p([b+2,n])}$
and hence
%%\vspace{-0.1cm}
$$\frac{p([a,n])}{p(a)} \geq \frac{p([b+1,n])}{p(b+1)}.$$
Thus we obtain
%%\vspace{-0.1cm}
$$p(b+1)    \geq \frac{p([b+1,n])}{p([a, n])}p(a) =  \frac{1}{1+\eta} p(a)$$
%%\vspace{-0.1cm}
as desired.
%Since $p(R) \leq p([j+1, n]) \leq p([i+1, n]) \leq p(I) + p(R) = (1+\eta)
%p(R)$,
%\[ \frac{p(i)}{1+\eta} \leq p(j) . \]

%Now define $\tilde p(i) = \min\{p(j) \mid i\leq j\leq b\}$ for $i \in [a,b]$.
%Clearly $\tilde p(i)$ is non-decreasing for $i\in [a,b]$, and $\tilde p(i) \leq
%p(i) \leq (1+\eta) \tilde p(i)$.
%Consequently, $\tilde p(I)\leq p(I)\leq (1+\eta) \tilde p(I)$, implying
%\[ p_I(i) \leq (1+\eta)\tilde p_I(i) \qquad\text{and}\qquad \tilde p_I(i) \leq
%(1+\eta) p_I(i). \]
%Thus
%\[ |p_I(i) - \tilde p_I(i)| \leq \eta\cdot \min\{p_I(i), \tilde p_I(i)\}. \]
%Summing the last inequality over $i\in [a,b]$, and using $\sum_i \min\{c_i,
%  d_i\} \leq \min\{\sum_i c_i, \sum_i d_i\}$, we get $\norm{p_I - \tilde p_I}_1
%\leq \eta$.
\end{proof}

Let $\partit Q = \{I_1, I_2, \dots, I_{|\partit Q|}\}$, with $I_i = [a_i, b_i]$, $1 \leq i \leq |\partit Q|$, where $a_i < a_{i+1}$.
Let $\partit Q' = \{I_i \in \partit Q : p(a_i) > p(a_{i+1})\}$
and $\partit Q'' = \{I_i \in \partit Q : p(a_i) \leq p(a_{i + 1})\}$.
Thus, $\partit Q'$ consists of those intervals $I$ in $\partit Q$ which are
such that the following interval's initial value is significantly
smaller than the initial value of $I$, and $\partit Q''$
consists of those $I \in \partit Q$ for which the following
interval's initial value is significantly larger than the initial
value of $I$.
We also denote $R_i = [a_{i+1}, n]$.
For convenience, we also let $a_{|\partit Q| + 1} = b_{|\partit Q|} + 1$.

We first bound the ``total multiplicative
decrease in $p$'' across all intervals
in $\partit Q'$:
\begin{lemma} \label{lem:prod}
We have
$
\littleprod_{I_i \in \partit Q'} \frac{p(a_i)}{p(a_{i+1})} \leq \frac{8}{\epsilon}.
$
\end{lemma}
\begin{proof}
Observation \ref{obs:intervalmass} implies that
the total probability mass $p(I \cup I')$
on intervals $I$ and $I'$ is  at least $\epsilon/8$.
We thus have
%%\vspace{-0.1cm}
$$
\littleprod_{I_i \in \partit Q'} \frac{p(a_i)}{p(a_{i+1})}
\leq
\littleprod_{I_i \in \partit Q'}\frac{p(I_i) + p(R_i)}{p(R_i)}
\leq
\littleprod_{I_i \in \partit Q}\frac{p(I_i) +p(R_i)}{p(R_i)}
=
\frac{p(I_1)+p(R_1)}{p(R_{|\partit Q|})}
\leq
{\frac 1 {\eps/8}},$$
where the first inequality follows from
Lemma \ref{lemma:mhr-monotone}, the second inequality is self-evident,
the equality follows from the telescoping product, and the final
inequality is because $p(I \cup I') \geq \eps/8.$
\end{proof}

At this point we can bound the number of intervals produced
by \textsc{Decompose-MHR}:
\begin{lemma}\label{lem:mhr_intervals}
Step~4 of Algorithm \textsc{Decompose-MHR} adds at most
$O(\log(n/\epsilon)/\epsilon)$ intervals to $\partit Q$.
\end{lemma}
%%\vspace{-0.1cm}
\begin{proof}
We first bound the number of intervals in $\partit Q'$. Let the
intervals in $\partit Q'$ be $I'_1, I'_2, \dots I'_{|\partit Q'|}$,
where $I'_j = [a'_j,b'_j]$ and $a_1' > a_2' > \dots > a_{|\partit Q'|}'$.
Observation \ref{obs:intervalmass} implies that
the total probability mass $p(I \cup I')$ is  at least $\epsilon/8$.
Hence, $p([b_1'+1, n])$  is at least $\epsilon / 8$ and
we have $p(R'_1) \geq \eps/8$. For $j \geq 1$ it holds
%%\vspace{-0.1cm}
\begin{equation} \label{eq:q}
p(R_{j}') \geq (\eps/8)\left(1+ \eps/8 \right)^{j-1}.
%%\vspace{-0.1cm}
\end{equation}
%%\vspace{-0.1cm}
\noindent Consequently the number of intervals in $\partit Q'$ is bounded
by $O(\log(1/\epsilon)/\epsilon)$.

Now we bound the number of intervals in $\partit Q''$.
We consider the value of $\prod_{I_i \in \partit Q} \frac{p(a_{i+1})}{p(a_{i})}$:
%%\vspace{-0.1cm}
$$
\frac{p(a_{|\partit Q| + 1})}{p(a_1)} = \littleprod_{I_i \in \partit Q} \frac{p(a_{i+1})}{p(a_{i})} = \littleprod_{I_i \in \partit Q''} \frac{p(a_{i+1})}{p(a_{i})}
\cdot \littleprod_{I_i \in \partit Q'} \frac{p(a_{i+1})}{p(a_{i})}.
$$
%%\vspace{-0.1cm}
Since $p(a_{|\partit Q| + 1}) \leq 1$ and $p(a_1) \geq \eps/(4n)$,
the above is at most $4n/\eps$; using Lemma~\ref{lem:prod},
we get that
%%\vspace{-0.1cm}
$$
\littleprod_{I_i \in \partit Q''} \frac{p(a_{i+1})}{p(a_{i})} \leq (4n/\eps) \cdot (8/\eps) = (32n/\eps^2).
$$
%%\vspace{-0.1cm}
\noindent On the other hand, for every $I_i \in \partit Q''$ we have
that $\frac{p(a_{i+1})}{p(a_i)} \geq (1+\eps/8)$.
Consequently there can be at
most $O((1/\eps) \log(n/\epsilon))$ intervals in $\partit Q''$,
and the proof is complete.
\end{proof}
%%\vspace{-0.1cm}

It remains only to show that $\partit P$ is actually a flat
decomposition of $p$:
%%\vspace{-0.1cm}
\begin{theorem}
\label{thm:mhr-dec}
Algorithm \textsc{Decompose-MHR} outputs a partition $\partit P$
of $[n]$ that is $(p, \eps, \log(n/\eps)/\eps))$-flat.
\end{theorem}
\begin{proof} 
Lemma \ref{lem:mhr_intervals} shows that $\partit P$ contains at most
$O(\log(n / \epsilon)/\epsilon)$ intervals, so
it suffices to argue that $\dtv \left( p,p^{\flattened(\partit P)} \right) \leq
\eps.$

We first consider the two rightmost intervals $I$ and $I'.$
If $|I|=1$ then clearly $\dtv \left( p^I , p^{\flattened(I)} \right)=0$, and if
$|I|>1$ then $p(I) \leq \eps/8$ and consequently
$\dtv \left( p^I , p^{\flattened(I)} \right) \leq \eps/8.$
Identical reasoning applies to $I'.$
For the leftmost interval $I''$, we have that
$p(I'') \leq \eps/4$, so
$\dtv (p^{I''} , p^{\flattened(I'')}) \leq \eps/4.$
Thus, so far we have shown that the contribution to $\dtv\left(p, p^{\flattened({\partit P})} \right)$ from $I \cup I' \cup I''$ is at most
$\eps/2.$

Now for each interval $I'''$ in $\mathcal{Q}$, we have
%%\vspace{-0.1cm}
$$ \max_{i, j\in I'''} \frac{p(i)}{ p(j)} \leq
\left(1+ \eps/8\right)^2 = 1 +\eps/4 + \eps^2/64.$$
%%%\vspace{-0.05cm}
\noindent Since the total probability mass on intervals $I$ and $I'$ is at
least $\epsilon/8$ by Observation \ref{obs:intervalmass},
the total probability mass on intervals in $\partit Q$ is at
most $1-\eps/8$.  An easy calculation
using Observation \ref{obs:multi-uniform} shows that the total
contribution to $\dtv(p,p^{\flattened({\partit P})})$
from intervals in ${\cal Q}$ is at most $\eps / 4$,
and the theorem is proved.
\end{proof}

Applying Corollary~\ref{cor:main}, we get our main learning result
for mixtures of MHR distributions:

\begin{corollary} [see Theorem~\ref{thm:mhr-informal}]
\label{cor:learn-mhr-mixture}
Let $p$ be any $k$-mixture of MHR distributions over $[n]$.
There is an algorithm \textsc{Learn-MHR-Mixture}$(p,k,\eps,\delta)$
that draws $O(k\log(n/\eps)/\eps^4 + \log(1/\delta)/\eps^2)$ samples from
$p$ and with probability at least $1-\delta$ outputs a
distribution $h$ such that $\dtv(p,h) \leq \eps.$
Its running time is
$\tilde{O}((\log n)^2 \cdot (k\log(1/\eps)/\eps^4 + \log(1/\delta)/\eps^2))$
bit operations.
\end{corollary}

\smallskip

\noindent
{\bf Lower bounds.}
By adapting a lower bound of~\cite{Birge:87} (for monotone distributions over a continuous interval)
it can be shown that for $\eps \geq 1/n^{\Omega(1)}$,
any algorithm for learning a monotone distribution over $[n]$
to accuracy $\eps$ must use $\Omega(\log(n)/\eps^3)$
samples.  We may consider a uniform mixture of $k$
component distributions
where the $i$-th distribution in the mixture is supported on
and monotone non-decreasing over $[1+(i-1)n/k,in/k]$.
Each component distribution is MHR (over the entire domain).
The same argument as in the log-concave case implies that,
for $k \leq n^{1-\Omega(1)}$ and $\eps \geq 1/n^{\Omega(1)}$,
any algorithm for learning a mixture of $k$ MHR distributions to
accuracy $\eps$ must use $\Omega(k \log(n)/\eps^3)$ samples.

\section{Learning mixtures of unimodal and $t$-modal distributions}
\label{sec:unimodal}

In this section we apply our general method from Section~\ref{sec:main}
to learn mixtures of unimodal (and, more generally, $t$-modal)
distributions over $[n].$  Here our task is quite easy because
of a result of L. Birg\'{e} \cite{Birge:87b}
which essentially
provides us with the desired flat decompositions.
\footnote{We note that Birg\'{e}'s structural result was obtained
as part of an efficient learning algorithm for monotone distributions;
Birg\'{e} subsequently gave an efficient learning algorithm
for unimodal distributions \cite{Birge:97}.  However, we are not aware
of work prior to ours on learning \emph{mixtures} of unimodal or
$t$-modal distributions.}

We begin by defining unimodal and $t$-modal distributions over $[n]$:

\begin{definition}
A distribution $p$ over $[n]$ is \emph{unimodal}
if there exists $i \in [n]$ such that $p$ is non-decreasing over $[1, i]$ and
non-increasing over $[i, n]$.  For $t > 1$, distribution $p$ over
$[n]$ is \emph{$t$-modal} if there is a partition of $[n]$ into $t$
intervals $I_1,\dots,I_t$ such that the sub-distributions $p^{I_1},
\dots,p^{I_t}$ are unimodal.
\end{definition}

%\medskip

By adapting a construction of Birg\'{e} (proved in \cite{Birge:87b}
for distributions over the continuous real line) to the discrete
domain $[n]$, \cite{DDSVV13:soda} established the
following:

%more important is to describe the oblivious decomposition that
%\cite{Birge:87} gives for learning a single monotone distribution,
%since that's what our
%algorithm for mixtures will use.)

%\begin{lemma}\label{DDS12}[DDS12]
%Let $p$ be any distribution over [n], and $\mathcal{I} = \{I_{i}\}_{i=1}^a$  be a $(p, \epsilon, a)$-flat decomposition, and let $\mathcal{J} = \{I_{i}\}_{i=1}^b$ be a refinement of $\mathcal{I}$. Then $\mathcal{J}$ is a $(p, 2\epsilon, b)$-flat decomposition of $[n]$.
%\end{lemma}

\begin{theorem} [{\cite[Theorem~5]{DDSVV13:soda}}] \label{thm:birge}
Let $p$ be any monotone distribution (either non-increasing
or non-decreasing) over $[n]$.
Then $p$ is $(\eps,O(\log(n)/\eps))$-flat.
%
%
%Fix any $n \in \mathbb{Z}^+$ and $\epsilon>0$. The partition $\mathcal{I} := \{I_i\}_{i=1}^\ell$ of $[n]$, in which the $j$-th interval has size $\lfloor(1+\epsilon)^j\rfloor$ has the following properties: $\ell = O((1/\epsilon) \cdot \log(\epsilon \cdot n  +1))$, and $\mathcal{I}$ is a $(p, O(\epsilon), \ell)$-flat decomposition of $[n]$ for any non-increasing distribution $p$ over $[n]$, for some constant $\lambda$.
\end{theorem}

We note that it can be shown (using the same construction that is used
in the $\Omega(\log(n)/\eps^3)$ sample complexity lower bound of
\cite{Birge:87} for learning monotone distributions) that
$O(\log(n)/\eps)$ is the best possible
bound for the number of intervals required in
Theoorem~\ref{thm:birge}.

An immediate consequence of Theorem~\ref{thm:birge} is that any
unimodal distribution over $[n]$ is $(\eps,O(\log(n)/\eps))$-flat, and
any $t$-modal distribution over $[n]$ is $(\eps,O(t \cdot \log(n)/\eps))$-flat.
Using Corollary~\ref{cor:main} we thus obtain the following results
for learning mixtures of unimodal or $t$-modal distributions:

\begin{corollary} [see Theorem~\ref{thm:unimodal-informal}]
\label{cor:mixture-unimidal}
For any $t \geq 1$,
let $p$ be any $k$-mixture of $t$-modal distributions over $[n]$.
There is an algorithm \textsc{Learn-Multi-modal-Mixture}$(p,k,t,\eps,\delta)$
that draws $O(kt\log(n)/\eps^4 + \log(1/\delta)/\eps^2)$ samples from
$p$ and with probability at least $1-\delta$ outputs a
distribution $h$ such that $\dtv(p,h) \leq \eps.$
Its running time is
$\tilde{O}(\log(n) \cdot (kt\log(n)/\eps^4 + \log(1/\delta)/\eps^2))$
bit operations.
%Algorithm $\textsc{Learn-Unknown-Decomposition}$ can learn $k$-mixture of unimodal distributions using $O(\frac{k\log(\epsilon \cdot n + 1)}{\epsilon^4})$ samples.
\end{corollary}

\noindent {\bf Lower bounds.}
The lower bound arguments we gave for mixtures of
MHR distributions (which are based
on Birg\'{e}'s lower bounds for learning monotone distributions) apply
unchanged for mixtures of unimodal distributions, since every
distribution which is supported on and monotone non-decreasing over
$[1 + (i-1)n/k,in/k]$ is unimodal over $[n].$

\section{Conclusions and future work} \label{sec:conclusions}

This work introduces a simple general approach to learning mixtures of
``structured'' distributions over discrete domains. We illustrate the usefulness of our approach
by showing it yields nearly optimal algorithms for learning mixtures of natural and well-studied
classes (log-concave, MHR and unimodal) and in the process we establish novel structural properties
of these classes.

Are there any other natural distribution classes for which our general framework is applicable?
We suspect so. At the technical level, the linear dependence on the parameters $k$ and $t$ in the sample complexity of Theorem~\ref{thm:general}
is optimal (up to constant factors). It would be interesting to improve the dependence on $1/\eps$ from cubic down
to quadratic (which would be best possible) with an efficient algorithm.

\bibliography{allrefs}
\bibliographystyle{alpha}

\appendix

\section{Proof of Proposition~\ref{obs:generalmixtures}}
\label{ap:generalmixtures}
At a high level, the algorithm $A'$ works by drawing a large
set of samples from the target mixture and trying all possible
ways of partitioning the sample into $k$ disjoint subsamples.  For each
partition of the sample it runs
algorithm $A$ over each subsample and combines the resulting hypothesis
distributions (guessing the mixture weights) to obtain a hypothesis mixture
distribution.  Finally, a ``hypothesis testing'' procedure is used to identify
a high-accuracy hypothesis from the collection of all
hypotheses distributions obtained in this way.

More precisely, let $p$ denote the unknown target $k$-mixture
of distributions from $\mathfrak{C}$.
Algorithm $A'$ works as follows:

\begin{enumerate}

\item Draw a sample $S$ of
 $M = \tilde{O}(k/\eps) \cdot m(n,\eps/20) \cdot \log(5k/\delta)$
samples from $p$.

\item For each possible way of partitioning $S$ into $k$
disjoint subsamples $\bar{S}=(S_1,\dots,S_k)$ such that each
$|S_i| \geq m(n,\eps/20) \cdot \log(5k/\delta)$, run algorithm~$A$
a total of $k$ times, using $S_i$ as the input sample for the $i$-th run,
to obtain hypothesis distributions $h^{\bar{S}}_1,\dots,h^{\bar{S}}_k$.
For each vector $\mu = (\mu_1,\dots,\mu_k)$ of non-negative mixing weights
that sum to 1 and satisfy $\mu_i = $ (integer)$\cdot \eps/(20k)$,
let $h^{\bar{S}}_\mu$ be the mixture distribution $\sum_{i=1}^k
\mu_i h^{\bar{S}}_i.$

\item Draw $M' = O(M \log k + k \log(k/\eps)) \cdot \log(5/\delta) / \eps^2$
samples from $p$ and use them to
run the ``hypothesis testing'' routine described in Lemma~11 of
\cite{DDS12stoc} over all hypotheses $h^{\bar{S}}_\mu$ obtained in the
previous step.  Output the hypothesis distribution that this
routine outputs.

\end{enumerate}

We now proceed with the analysis of the algorithm.
Let $p = \sum_{i=1}^k \kappa_i p_i$ denote the target $k$-mixture,
where $\kappa_1,\dots,\kappa_k$ are the mixing weights and
$p_1,\dots,p_k$ are the components.  Without loss of generality
we may assume that $i=1,\dots,\ell$ are the components
such that the mixing weights $\kappa_1,\dots,\kappa_\ell$
are at least $\eps/(20k).$
A standard ``balls in bins'' analysis (see \cite{NS:60})
implies that with probability at least $1 - \delta/5$
the sample $S$ contains at least $m(n,\eps/20) \cdot \log(5k/\delta)$
draws from each component $p_1,\dots,p_\ell$; we assume going forth
that this is indeed the case.  Thus there will be some partition
$\bar{S}=(S_1,\dots,S_k)$ which is such that each $S_i$ with
$1 \leq i \leq \ell$ consists entirely of samples drawn from the component
$p_i$.  For this $\bar{S}$, we have that with failure probability at most $
(\delta/(5k)) \cdot k \leq \delta/5$, each hypothesis distribution
$h^{\bar{S}}_i$ for $1 \leq i \leq \ell$ satisfies
$\dtv(p_i,h^{\bar{S}}_i) \leq \eps/20.$
Now let $\mu^\star$ denote the vector of hypothesis mixing weights
(as described in Step~2) that has $|\mu^\star_i - \kappa_i| \leq
\eps/(20k)$ for all $i=1,\dots,k.$ It
is not difficult to show that the hypothesis mixture distribution
$h^\star = h^{\bar{S}}_{\mu^\star}$ satisfies $\dtv(h^\star,p) \leq
3 \eps / 20 < \eps/6$, where $\eps/20$ comes from the
errors $\dtv(p_i,h^{\bar{S}}_i)$ for $i \leq \ell$,
$\eps/20$ comes from the inaccuracy in the mixing weights, and
$\eps/20$ comes from the (at most $k$) components $p_j$ with $j >
\ell$ that each have mixing weight at most $\eps/(20k).$

Thus we have established that there is at least one hypothesis
distribution $h^\star$ among the $h^{\bar{S}}_\mu$'s that has $\dtv(p,h^\star)
\leq \eps/6.$  There are at most $N = k^M \cdot (20k/\eps)^k$
hypotheses $h^{\bar{S}}_\mu$ generated in Step~2, so the
algorithm of Lemma~11 of \cite{DDS12stoc} requires $O(\log N)
\log(5/\delta) /\eps^2 \leq M'$ samples, and with probability
at least $1 - \delta/5$ it outputs a hypothesis distribution $h$
that has $\dtv(p,h) \leq \eps.$
The overall probability of outputting an $\eps$-accurate hypothesis
is at least $1-\delta$, and the proposition is proved.

\end{document}